\documentclass[twoside,11pt]{article}

\usepackage{jmlr2e}

\usepackage[utf8]{inputenc}
\usepackage{amsmath}
\usepackage{thmtools, thm-restate}
\usepackage{bm}
\usepackage{graphicx}
\usepackage{xcolor}
\usepackage{mdframed}
\usepackage{bbold}
\usepackage{subfiles}
\usepackage{caption}
\usepackage{subcaption}
\usepackage{natbib}
\usepackage{tikz}
\usepackage{url}
\usepackage{amsfonts}
\usepackage{soul}
\usepackage[ruled,vlined]{algorithm2e}
\usepackage{mdframed}
\usetikzlibrary{positioning}

\newcommand*{\fullref}[1]{\hyperref[{#1}]{\autoref*{#1} \nameref*{#1}}}

\usepackage{ifthen}
\newcommand{\conref}[1]{%
  \ifthenelse{\equal{#1}{\string 1}}
    {#1$^{\text{st}}$}
    {
    \ifthenelse{\equal{#1}{\string 2}}
        {#1$^{\text{nd}}$}
        {
            \ifthenelse{\equal{#1}{\string 3}}
            {#1$^{\text{rd}}$}
            {#1$^{\text{th}}$}%
        }%
    }%
}

\firstpageno{1}

\begin{document}

\title{Conceptual capacity and effective complexity of neural networks}

\author{\name Lech Szymanski \email lech.szymanski@otago.ac.nz \\
        \name Brendan McCane \email brendan.mccane@otago.ac.nz \\
        \name Craig Atkinson \\
       \addr Department of Computer Science\\
       University of Otago\\
       Dunedin, 9016, New Zealand
      }

\editor{Kevin Murphy and Bernhard Sch{\"o}lkopf}

\maketitle


\begin{abstract}
We propose a complexity measure of a neural network mapping function based on the diversity of the set of tangent spaces from different inputs.  Treating each tangent space as a linear PAC concept we use  an entropy-based measure of the bundle of concepts in order to estimate the \textit{conceptual capacity} of the network.  The theoretical maximal capacity of a ReLU network is equivalent to the number of its neurons.  In practice however, due to correlations between neuron activities within the network, the actual capacity can be remarkably small, even for very big networks.  Empirical evaluations show that this new measure is correlated with the complexity of the mapping function and thus the generalisation capabilities of the corresponding network.  It captures the effective, as oppose to the theoretical, complexity of the network function.  We also showcase some uses of the proposed measure for analysis and comparison of trained neural network models.
\end{abstract}


\begin{keywords}
Deep learning, learning theory, complexity measures
\end{keywords}

\maketitle


\section{Introduction}

Statistical learning theory relates theoretical complexity of a hypothesis space to an upper bound on generalisation error of a hypothesis selected from that space.   The most recent   \textit{nearly-tight} estimate of the upper bound on VC-dimension of a piece-wise linear neural network puts its theoretical complexity at $O(WL\log(W))$, where $W$ is the total number of parameters and $L$ the number of layers in the network \citep{Bartlett.etal:2019}.  By this estimate, the bigger and deeper the network, the less likely it is to generalise well.  In practice, the exact opposite is observed.

The problem with theoretical complexity might be that taking the entire hypothesis space into account is too conservative.  \citet{Choromanska.etal:2015} have shown that as the size of a ReLU network grows, it becomes increasingly likely that a good local minima will be found in training and not a global one that is more prone to overfitting. Further, the Lottery Ticket Hypothesis \citep{Frankle.etal:2019} suggests that larger networks do not necessarily utilise all their computational resources, tending to \textit{select} and adapt one of many smaller subnetworks as the solution to the problem at hand.  Also, it has been well established that despite potential for \textit{memorisation} of arbitrary input-output mappings during training \citep{zhang.etal:2017}, when possible, overparameterised neural networks \textit{discover} common patterns that lead to good generalisations \citep{devansh.etal:2017}.  It seems evident that big deep neural networks do not explore the entirety of their hypothesis space and do not reach their theoretical complexity if there is no need for it.  

In this work we propose a measure of \textit{effective complexity}, which relates the complexity of an individual hypothesis in the context of training data, as opposed to that of its hypothesis space.  In  neural network terms, this means abstracting away the architecture and measuring the complexity of a specific input-output mapping produced by a given set of parameters in that architecture.  For instance, the effective complexity of a ReLU neural network, however large, with parameters such that for every valid input all the neurons operate in their linear region, is equivalent to effective complexity of a linear model.  The proposed method characterises a neural network through first-order approximations of its mapping function.  These approximations are treated as linear concepts, each providing an individual input-output mapping.  We measure pairwise similarity between concepts and estimate the complexity of the overall mapping through an entropy-based measure of the concept bundle.  We refer to the resulting complexity measure as \textit{conceptual capacity}.


\section{Related work}

The jump off point for our work is the work of \citet{Neyshabur.etal:2017}, which  experimentally evaluates different measures of neural network complexity\footnote{\citet{Neyshabur.etal:2017} make no distinction between the theoretical and effective complexity.} with respect to three conditions that they deem are required for explanation of the good generalisation tendencies of overparameterised neural networks.  The conditions are:
\textit{
\begin{enumerate}
\item ``Networks learned using [datasets with] real labels (and which generalize well) [must exhibit] lower complexity, under the suggested measure, than those learned using random labels (and which obviously do not generalize well)."\label{cond:random}
\item ``(...) The complexity measure [can] decrease [even] as we increase the number of [parameters in the network]."\label{cond:bignet}
\item ``We (...) expect to see a correlation between the complexity measure and generalization ability among zero-training error models."\label{cond:gen}
\end{enumerate}
}

The theoretical complexity bound of $O(WL\log(W))$ given by \citet{Bartlett.etal:2019} clearly does not satisfy the first two conditions, as it does not take the state of the network after training into account, and is proportional to network size.  Effective complexity measures based on norms of network weight matrices, proposed by \citet{Neyshabur15.etal:2015}, do not satisfy the \conref{\ref{cond:bignet}} condition as they still increase with the network size, though not as quickly as the theoretical complexity bounds.  Complexity by sharpness, introduced by \citet{Nitish.etal:2017}, frees itself from dependency on network size by equating complexity to the degree of changes in network output in relation to perturbations of its weights.   However, on its own, it does not not satisfy the \conref{\ref{cond:random}} condition.  This can be alleviated by supplementing sharpness with PAC-Bayes measures, as proposed by \citet{Neyshabur.etal:2017}, but the new complexity still fails on condition \ref{cond:gen}.  In the end \citet{Neyshabur.etal:2017} finds that neither of the existing measures satisfies all three of their proposed conditions.  The effective complexity measure we propose in this paper does.

Our work builds on the notion of switched linear projection from \citet{Szymanski.etal:2020} and that of an ``active state" of a ReLU neural network.  The principle guiding our analysis is that internally, a neural network responds to a given input by switching some of its neurons ``off".  While these states might be unique for every training input, the correlations between the output of hidden neurons, arising as a result of training, ensure that some of those states produce similar functions.  We formalise and generalise the notion of these states combining ideas from differential geometry and quantum mechanics.  From the former, we equate a differential tangent bundle of a function to a set of linear concepts.  We use the fact that the internal representation of a neural network arising from a specific configuration of its parameter values is capable of giving rise to a limited diversity of these concepts.  Hence, we turn to von Neumann entropy of the similarity matrix of a sample of differential tangent concepts to measure this diversity, which we refer to as conceptual capacity.  In plain terms, we treat the neural network function as an interpolation of a set of simple concepts, and measure the size of the concept space hidden within the internal representation of a network of a particular architecture and values of its weights.  


\section{Conceptual capacity}

\newcommand\TB{\mathcal{T} \mathcal{M}}
\newcommand\TS{\mathcal{T}_x \mathcal{M}}
\newcommand\TfS{\mathcal{T}_{f(x)} \mathcal{N}}
\newcommand\Man{\mathcal{M}}
\newcommand\Nan{\mathcal{N}}
\newcommand\G{\tau}
\newcommand\xdim{D}
\newcommand\ydim{U}
\newcommand\dec[1]{\mathcal{D}\big ( {#1} \big )}
\newcommand{\Cset}[2]{\mathcal{C}_{#1#2}}
\newcommand\simf{s}
\newcommand\Simf{S}
\newcommand\TBen{K}
\newcommand\CD{\text{CD}}
\newcommand\InDom{[-1,1]}
\newcommand\OutDom{[-1,1]}
\newcommand\Partition{\mathcal{P}}
\newcommand\Sim{\text{sim}}

\begin{figure}
    \centering
    \includegraphics[width=0.35\textwidth]{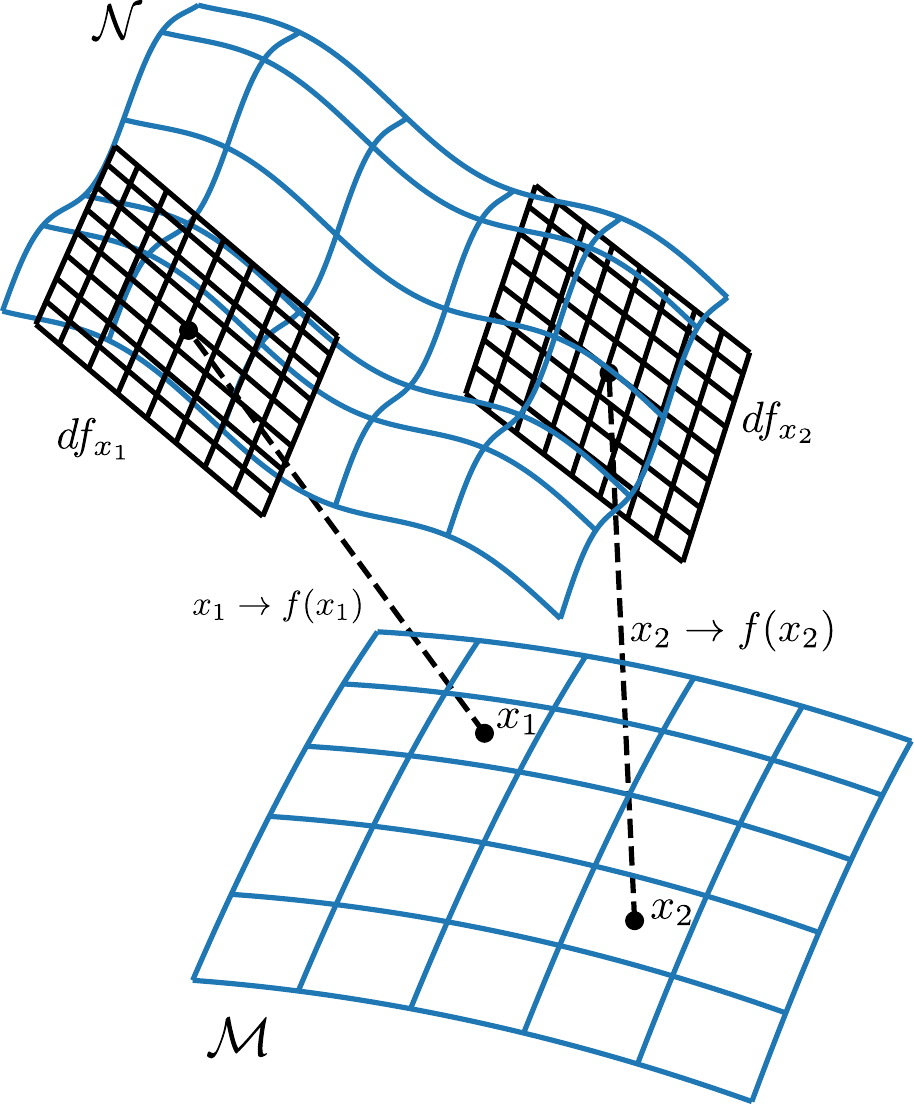}
    \caption{Illustration of $f:\mathcal{M}\to \mathcal{N}$ with tangent bundle map of two differentials at two points $x_1$ and $x_2$; $\mathcal{M}$ is the manifold of network inputs, $\mathcal{N}$ is the manifold of outputs, $df_{x_n}$ is the first order approximation of $f(x_n)$ at $x_n$ -- a concept that is part of the input-output mapping.}
    \label{fig:tangents}
\end{figure}

We associate with a neural network an almost everywhere (a.e.) differentiable function $f: \InDom^\xdim \rightarrow {\OutDom}^\ydim$\footnote{We assume without loss of generality, that the network maps to $\OutDom^\ydim$.}. We assume that the data of interest lie on a $k$-dimensional Riemannian manifold, $\Man$ where $k \le \xdim$. We set the Riemannian metric, $g$, to be the metric induced by the Euclidean metric in $\mathcal{R}^\xdim$.
We are interested in the behaviour of the network (or $f$) on $\Man$ only, so that we are interested in the function $f: \Man \rightarrow \Nan$, where $\Nan$ is the output manifold immediately prior to the decision function. For example, in the case of binary classification, $\Nan \subseteq \OutDom$, and the decision function would be a threshold on the output value. 

\subsection{Tangent bundle of a function}

We begin with a discretisation of $f$'s bundle map at training data points. For this we need some notation from differential geometry \citep{Lee:2012}.

\begin{definition}[Tangent Space $\TS$]\label{def:tspace}
Given a smooth $k$-dimensional manifold, $\mathcal{M}$, $\TS$ is the $k$-dimensional real vector space tangent to $\mathcal{M}$ at point $x$.
\end{definition}

\begin{definition}[Tangent Bundle $\TB$]\label{def:tbundle}
The tangent bundle, $\TB$, of a manifold $\mathcal{M}$ is given by the disjoint union of the tangent spaces of $\mathcal{M}$:
\begin{equation*}
    \TB = \{(x,y) | x \in \mathcal{M}, y \in \TS \}
\end{equation*}
\end{definition}

\begin{definition}[Differential $df_a$]\label{def:diff}
The differential of a smooth map, $f: \mathcal{M} \rightarrow \mathcal{N}$ is a linear map $df_x: \TS \rightarrow \TfS$.
\end{definition}

\noindent A differential $df_x$ is essentially a linear model of $f$ at point $x$.  For our purposes, $df_x$ is the hyperplane arising from the first order Taylor series approximation of $f$ at $x$ (computed as shown in Section \ref{sec:implement}).  


\begin{definition}[Tangent Bundle Map $df$]\label{def:bmap}
The differential of $f$ induces a tangent bundle map in the obvious way, $df: \TB \rightarrow \mathcal{TN}$.
\end{definition}

\noindent Figure \ref{fig:tangents} shows an illustration of a mapping from 2D manifold to a 2D manifold with differentials evaluated at two points.  It is worth nothing at this point the cardinally of $df$ with respect to some functions of interest.

\begin{lemma}\label{prop:tspacelin}
The cardinality of the tangent bundle map $df$ of a linear $f$ is
\begin{equation*}
|df|=1.
\end{equation*}
\end{lemma}
\begin{proof}
The first order Taylor polynomial of a linear function is the function itself, so that $df_x=f$ for all $x\in \mathcal{M}$ inducing a bundle consisting of single tangent space.  As a result $df=\{ f \}$ and therefore $|df|=1$.
\end{proof}

\begin{lemma}\label{prop:tspacerelu}
The cardinality of the tangent map $df$ of $f$ that is a ReLU network (without max pooling) with the total number of $K$ neurons\footnote{The number of neurons in a convolutional layer is equivalent to the size of its output map times the number of filters.} is:
\begin{equation*}
|df|\le 2^K.
\end{equation*}
\end{lemma}

\begin{proof}
This property follows from the switching projection argument from \citet{Szymanski.etal:2020}.  For given input $x$ some ReLU neurons will be in their ``off" state, producing output of zero, and others in their  ``on" state passing the activity of the neuron as a linear function.  Since the ``off" neurons do not contribute to the computation of the output, the subnetwork that participates in the computation of $f(x)$ consists solely of the ``on" neurons that are in their linear state.  This subnetwork is linear and by Lemma \ref{prop:tspacelin} gives rise to the same differential for any $x$ that activates that subnetwork.  Given $K$ neurons, there are $2^K$ possible linear subnetworks, and so there are at most $2^K$ different $df_x$'s.
\end{proof}

The $|df|$ of a ReLU network with max pooling is finite as well, but the computation of its upper bound is a bit more complicated, as a max pooling unit can switch between $1+hw$ states, where $h\times w$ is the height and width of its input window.  Average pooling does not affect $|df|$, since it provides a linear transformation.

\subsection{Concept similarity}

Next, we introduce a similarity score or kernel function on $df$, which we will use to evaluate $f$'s complexity.  We start in the continuous domain and then proceed to discrete approximations.

\begin{definition}[Decision function]
A decision function, $\dec{f(x)}$, over manifold $\Nan$, partitions $\Nan$ into $c$ measurable partitions, $\Partition_{f,i}$, where $\cup_{1\le i \le c} \Partition_{f,i} = \Nan$.
\end{definition}

\noindent Example decision functions include the hardlimit function for binary classification, or softmax for multi-class classification, though for softmax we would have a decision made from multiple outputs $\dec{f(x)}=\dec{\begin{bmatrix}f^{[1]}(x) & \hdots & f^{[U]}(x)\end{bmatrix}}$ where $f^{[u]}$ is the output function of neuron $u$.   Note that a single partition might consist of multiple disconnected sets, we only require that a partition is measurable.

\begin{definition}[Differential decision function]
A differential, $df_{a}$, creates a linear decision function, $\dec{df_{a}(x)}$, where $a,x \in M$.
\end{definition}

\noindent We refer to the decision function $\dec{df_{a}(x)}$ as a \textit{concept} induced by input $a$.  

\begin{definition}[Volume of a subset]
The volume, $V(S)$, of a subset of a Riemannian manifold, $S \subseteq \Nan$, is defined as $V(S) = \int_{S} dV_g$, where $dV_g$ is the usual Riemannian volume form.
\end{definition}

\noindent The similarity between two concepts is defined simply as the relative volume in which those decision functions agree.

\begin{definition}[Similarity measure]
\label{def-sim}
The similarity, {\normalfont $\Sim$}, between two decision functions $\dec{f(x)}$ and $\dec{g(x)}$, is:
{\normalfont 
\begin{equation*}
    \Sim\Big [\dec{f(x)}, \dec{g(x)}\Big ] = \frac{1}{V(\Nan)}\sum_i V(P_{f,i} \cap P_{g,i})
\end{equation*}
}
\end{definition}

\noindent Definition \ref{def-sim} has several nice properties:
    if the decision functions agree everywhere, then $\Sim=1$ and if they agree nowhere then $\Sim=0$;
    $\Sim$ is symmetric;
    $\Sim$ is a kernel over decision functions;
    one can define a distance by setting $d(x_1,x_2)=1-\Sim(x_1,x_2)$ or $d(x_1,x_2)=\frac{1}{\Sim(x_1,x_2)}$.
In the probabilistic framework we can think of the above similarity as the probability measure over the agreement between two concepts as defined in \citet{Kearns.etal:1994}:

\begin{equation}
    \Sim\Big [\dec{f(x)}, \dec{g(x)}\Big ]=P\Big (|\dec{f(x)}-\dec{g(x)}|<\epsilon \Big ) 
\end{equation}
for arbitrary small $\epsilon$.

\subsection{Entropy of the concept space}

We propose to measure the capacity of the concept space by summarising the similarities between all pairs of concepts in the space. From Lemma \ref{prop:tspacerelu}, for ReLU networks, we know that the number of concepts is bounded above and therefore we could, in principle, compute an exact quantity based on the network parameters only. However, we are only concerned in the complexity within the manifold of interest, and since we assume the training data are a representative sample, we approximate the complexity by measuring the pairwise similarity of concepts at training data points. For $N$ training points, this involves computing an $N\times N$ similarity matrix.

\begin{definition}
The similarity matrix of a sample of $N$ differential decision functions from $df$ is

{\normalfont 
{\scriptsize
\begin{equation*}
    \Simf_{df_N}=\begin{bmatrix}\Sim\Big [\dec{df_{a_1}(x)}, \dec{df_{a_1}(x)}\Big ] & \hdots & \Sim\Big [\dec{df_{a_1}(x)}, \dec{df_{a_N}(x)}\Big ] \\
\vdots & \ddots & \vdots \\
\Sim\Big [\dec{df_{a_N}(x)}, \dec{df_{a_1}(x)}\Big ]  & \hdots & \Sim\Big [\dec{df_{a_N}(x)}, \dec{df_{a_N}(x)}\Big ]
\end{bmatrix},
\end{equation*}
}
}
where $df_{a}\in df$.  
\end{definition}

\noindent In a loose analogy with quantum physics, we think of each concept as a pure quantum state, with the network being composed of a mixture of such states. The similarity matrix, suitably normalised, is equivalent to a quantum density matrix of mixed states \citep{wilde2017quantum}, for which the natural measure of complexity is the von Neumann entropy. We also note that von Neumann entropy has a recent history in measuring the complexity of graphs (for example, \cite{han2012graph}). This leads to the following definition of conceptual capacity.

\begin{definition}[Conceptual capacity]\label{def:conentro}
Given eigenvalues $\{\lambda_1,\hdots,\lambda_N\}$ of similarity matrix $S_{df_N}$, conceptual capacity is defined as the von Neumann entropy of $S_{df_N}$
\begin{equation}\label{eqn:conentro}
\TBen_{df_N}=-\sum_i^N \widehat{\lambda}_i\log_2(\widehat{\lambda}_i)
\end{equation}
where $\widehat{\lambda}_i=\dfrac{\lambda_i}{\sum_{j=1}^N \lambda_j}$ is the normalised eigenvalue, and $0\log(0)=0$.  
\end{definition}

\noindent $\TBen_{df_N}$ is essentially an indicator of compression ratio of the differentials in the bundle $df_N$.  The rank of $\Simf_{df_N}$, and thus the number of non-zero $\lambda_i$'s, correspond to the number of independent concepts involved in mapping of the $N$ input points to the network's output.  The value of $\lambda_i$ gives the relative prominence of a concept within the mapping.  Thus, we expect a network that does more memorisation to use more concepts with more uniform distribution of the eigenvalues, which should result in a higher conceptual capacity.  

Lemmas \ref{prop:tspacelin} and \ref{prop:tspacerelu} established that cardinality of $|df|$ is finite for some (most notably ReLU) networks.  Since $|df|$ bounds the rank of $S_{df_N}$ for arbitrary $N$, we can derive the upper bound on $\TBen_{df}$.

\begin{lemma}\label{prop:hupper}
The upper bound on conceptual capacity is
\begin{equation*}
\TBen_{df}\le \log_2|df|
\end{equation*}
\end{lemma}

\begin{proof}
Since the rank of $S_{df_N}$ can be at most $|df|$, and uniform distribution has the largest entropy, then the largest possible $\TBen_{df}$, regardless of $N$, is based on the eigen decomposition of $S_{df_N}$ with $|df|$ non-zero eigenvalues where $\widehat{\lambda}_i=1/|df|$.  Conceptual capacity is therefore bounded by
\begin{align*}
\TBen_{df} & \le -\sum_{i}^{|df|}\frac{1}{|df|}\log_2\frac{1}{|df|}=\log_2|df|.
\end{align*}
\end{proof}

By Lemmas \ref{prop:tspacerelu} and \ref{prop:hupper} for a ReLU neural network $\TBen_{df}\le K$, where $K$ is the number of neurons.  Thus, another way of interpreting conceptual capacity is the minimum number of independent ReLU neurons (in a hypothetical abstract network) capable of producing an equivalent function mapping.  Each neuron in a ReLU network has the potential to double the conceptual capacity of a network, just like an additional bit doubles the representational capacity of a binary number.  The key departure from the binary representation analogy is that neurons in a neural network may not be independent.  In a given architecture, with the output of neurons depending on other neurons and the values of the weights on their connections, there is a degree of correlation between neural activity.  Hence, the need to measure the actual conceptual capacity of a neural network.  


\section{Implementation}\label{sec:implement}


\begin{algorithm}[t]
\SetAlgoLined
 Given $X=\{x_1,\hdots,x_N\}$ and $U$ output neurons\;
 \For{$i\mapsto 1,\hdots,N$}{
     \For{$u\mapsto 1,\hdots,U$}{
        Compute $df^{[u]}_{x_i}$\;
    }
  \For{$n\mapsto 1,\hdots,N$}{
    Compute label $h_{i,n}=\dec{\begin{bmatrix}df^{[1]}_{x_i}(x_n) & \hdots & df^{[U]}_{x_i}(x_n)\end{bmatrix}}$\;
  }
 }
 \For{$i\mapsto 1,\hdots,N$}{
    \For{$j\mapsto i,\hdots,N$}{
        $ \Simf_{df_N[i,j]} := 0$\;
        \For{$n\mapsto 1,\hdots,N$}{
            \If{$h_{i,n}=h_{j,n}$}{
                $\Simf_{df_N[i,j]} := \Simf_{df_N[i,j]}+1$\;
            }
        }
        $\Simf_{df_N[i,j]} := \Simf_{df_N[i,j]}/N$\;
        $\Simf_{df_N[j,i]} := \Simf_{df_N[i,j]}$\;
  }
 }
 Compute eigenvalues $\lambda_1,\hdots\lambda_N$ of $\Simf_{df_N}$\;
 Normalise $\hat{\lambda}_i := \lambda_i/\sum_{j=1}^N\lambda_j$\;
 $\TBen_{df_N} := -\sum_{i=1}^N\hat{\lambda}_i\log{\hat{\lambda}_i}$
 \caption{Computation of conceptual capacity}\label{alg:cc}
\end{algorithm}

The computation of $\TBen_{df_N}$ is described in Algorithm \ref{alg:cc} where the differential concept $df_{a}(x)$ is given by first order Taylor series approximation of given neurons output $f(x)$ at $a$,

\begin{equation}\label{eqn:dfx}
df_a(x)=f(a)+{\frac{\partial f(a)}{\partial x}}\cdot (x-a),
\end{equation}
where
\begin{equation*}
\frac{\partial f(a)}{\partial x}=\begin{bmatrix}\frac{\partial f(a)}{\partial x^{[1]}} & \hdots & \frac{\partial f(a)}{\partial x^{[D]}}\end{bmatrix}
\end{equation*}
is a vector, $x^{[i]}$ is the $i^{\text{th}}$ component of $x$, and ($\cdot$) is the dot product operator.
The function $f(x)$ is taken to be the output of a network's neuron prior to applying the activation function (for instance, before softmax).


\section{Evaluation}\label{sec:eval}

To evaluate the conceptual capacity as a measure of effective complexity, we trained various networks on benchmark datasets and analysed them based on the resulting $\TBen_{df_N}$ coupled with the performance on a test set.  Eigenvalue decomposition on very large similarity matrices proved too time consuming and memory demanding, and so for most evaluations we computed the mean $\TBen_{df_N}$ over several choices of different $N$ samples from the training set instead of one computation over the entire training set.  

\subsection{Convergence of $\TBen_{df_N}$}\label{sec:toy}

\begin{figure*}
    \centering
    \includegraphics[width=0.8\textwidth]{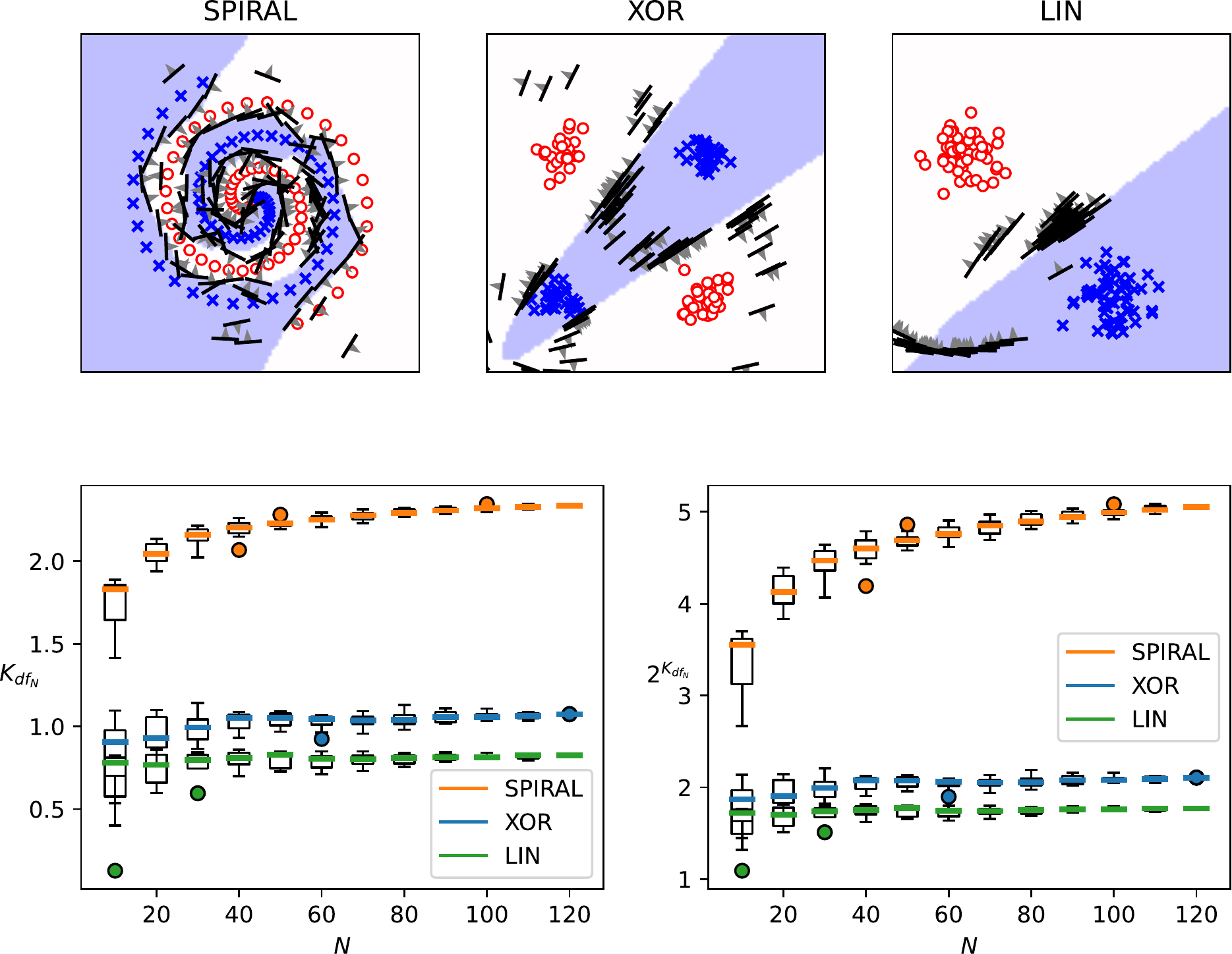}
    \caption{Visualisation of the function mapping (top row) provided by a single output fully connected 3 hidden layer network of 128 ReLU neurons each trained on 2D-input SPIRAL, XOR and LIN;  training points are shown as red o's and blue x's; the mapping function is depicted as white/blue shaded classification regions; the differentials $df_{x_i}$ from the bundle generated by all training points are shown as dashed line segments (indicating the boundary between half-spaces given by the hyperplane equation of the corresponding differential) with gray arrows (pointing toward the positive half-space); the conceptual capacity $\TBen_{df_N}$ and the corresponding number of independent concepts $2^{\TBen_{df_N}}$ are shown in the bottom row as box plots generated from 10 evaluations of different choices of random points from the dataset for each $N$.}\label{fig:toy}
\end{figure*}

For this first test we created three toy datasets for binary classification of different complexity: the SPIRAL, XOR and LIN (see top row of Figure \ref{fig:toy}).  Each dataset consisted of 120 points in 2D.  On each dataset we trained a single output fully connected neural network of 3 hidden layers and 128 ReLU neurons per layer (for details of the training see Appendix \ref{app:toy}).  The top row of Figure \ref{fig:toy} presents visualisations of the classification regions given by each trained network along with the differential bundle derived from all the 120 training points.  Each concept is depicted as a separating line. It is clear that even for these simple datasets we get distinct $df_{x_i}$ for each $x_i$, meaning $|df_N|\approx N$ for overparameterised networks.  However, the SPIRAL-trained network's bundle looks more diverse than that of the XOR-trained network, which in turns looks more diverse than that of the LIN-trained network.  

The bottom row of Figure \ref{fig:toy} shows a box plot of conceptual capacity $\TBen_{df_N}$ over 10 different random choices of $N$ training points (for cases where $N<120$) and the corresponding box plot of the number of corresponding concepts $2^{K_{df_N}}$.  Conceptual capacity increases with $N$ approaching a constant value, as most evident on the LIN example.  It might seem that $\TBen_{df_N}$ should monotonically increase with $N$.  However, as evident by $\TBen_{df_N}$ near $N=60$ on XOR, a drop in conceptual capacity can happen.  This occurs when the similarity matrix $\Simf_{df_N}$ over an expanded number of samples forming the input manifold maintains the same rank, but yields a less uniform distribution of $\lambda_i$'s in Equation \ref{eqn:conentro}, which could happen if additional points coincide with a previously considered input sample.  However, under i.i.d. assumption of the distribution of the $N$-point sample, we don't expect these drops to be very substantial.  

It is quite evident that conceptual capacity rates the SPIRAL-trained neural network as the most complex, then XOR, and then LIN, inline with the expectations based on the visualisation of their respective function maps.

\subsection{True vs. random labels}\label{sec:random}

\begin{figure*}
    \centering
    \includegraphics[width=0.8\textwidth]{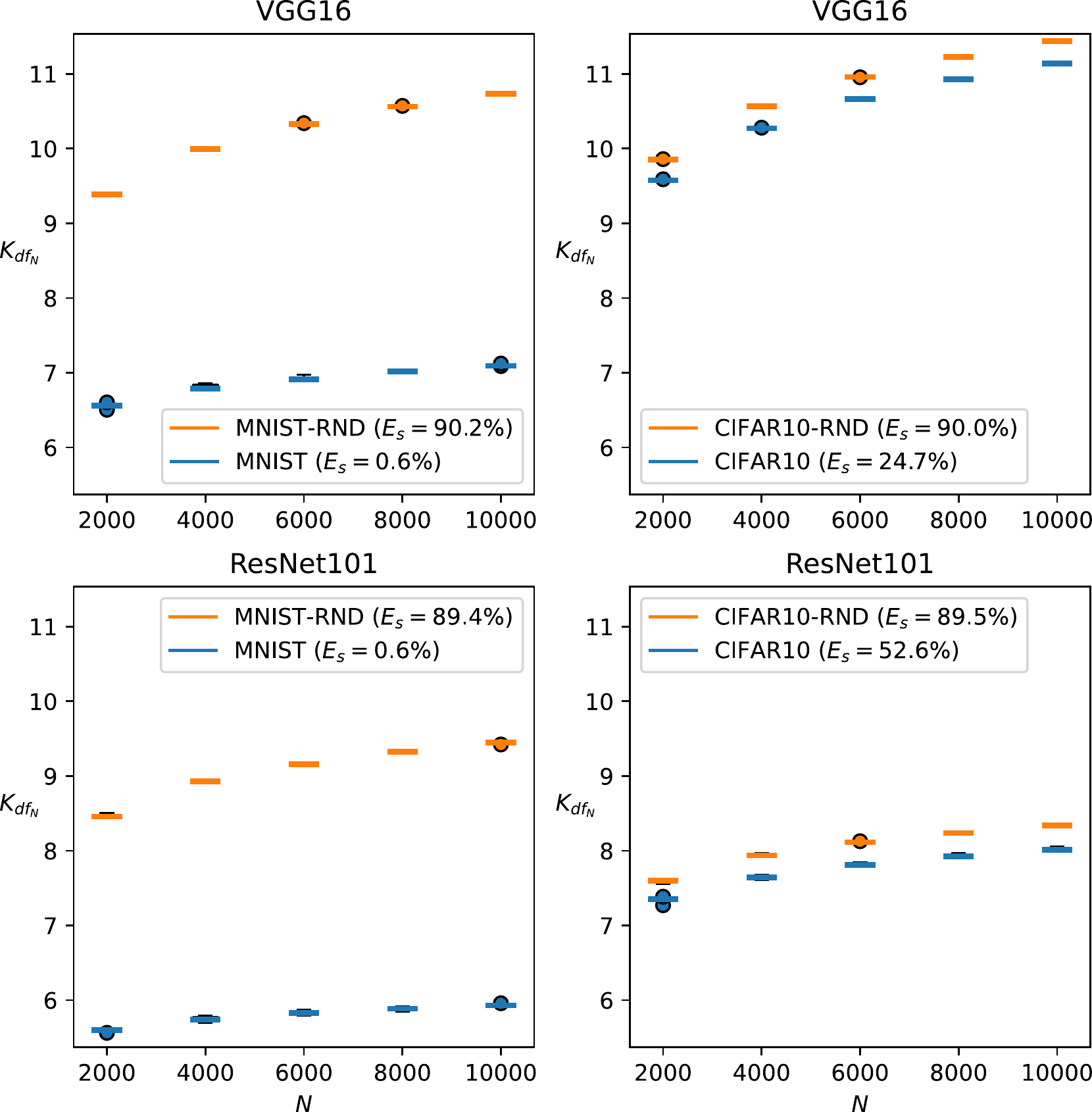}
    \caption{Conceptual capacity $\TBen_{df_N}$ over $N$ for VGG16 (top row) and ResNet101 (bottom row) trained on MNIST (left column) and CIFAR10 (right column) dataset vs. their random label variants, MNIST-RND and CIFAR10-RND; the box plots show the capacity from 5 evaluations over different choices of $N$ training samples; the training error for each network was 0, the percentage test error is shown in the legend as $E_s$.}\label{fig:random}
\end{figure*}

In this experiment, inspired by \cite{zhang.etal:2017} and \cite{devansh.etal:2017}, we trained VGG16 \citep{vgg16} and ResNet101 \citep{resnet101} networks to zero training error on the MNIST \citep{mnist} and CIFAR10 \citep{cifar10} datasets as well as on their randomised label variants, MNIST-RND and CIFAR10-RND (for details of the training see Appendix \ref{app:random}).  Figure \ref{fig:random} shows the results of the conceptual capacity analysis on the trained networks.  Since we did not calculate $\TBen_{df_N}$ based on the entire set of training points, we did 5 computations for given $N$, each time selecting different random sample of $N$ inputs from the training set.  The same sets of $N$ inputs were used for evaluations of $\TBen_{df_N}$ between the true and random label sets.  

In every case, the conceptual capacity of the true label--trained network is smaller than its random label--trained counterpart, thus satisfying \citet{Neyshabur.etal:2017}'s \conref{\ref{cond:random}} condition.  Granted that comparing conceptual capacity derived from different manifolds may not be entirely sensible (even after rescaling to the same image size CIFAR10 has three colour channels to MNIST's one), we nevertheless note an interesting difference between CIFAR10-trained and MNIST-trained networks.  For a given architecture, the conceptual capacity of the CIFAR10 network is higher than that of the MNIST one.  This is expected, since CIFAR10 is a \textit{tougher} to learn dataset.  What might come as a surprise, is that the conceptual capacity of CIFAR10-RND-trained ResNet101 is lower than that of MNIST-RND-trained ResNet101.  Based on our experience of ResNet101 being \textit{more resistant} to overtraining than VGG16, we concluded that ResNet101 finds it \textit{easier} to fit random labels to spurious patterns in the noisy background of CIFAR10 than in the clean MNIST.

\subsection{Different architectures}\label{sec:bignet}

\begin{figure*}
    \centering
    \includegraphics[width=0.8\textwidth]{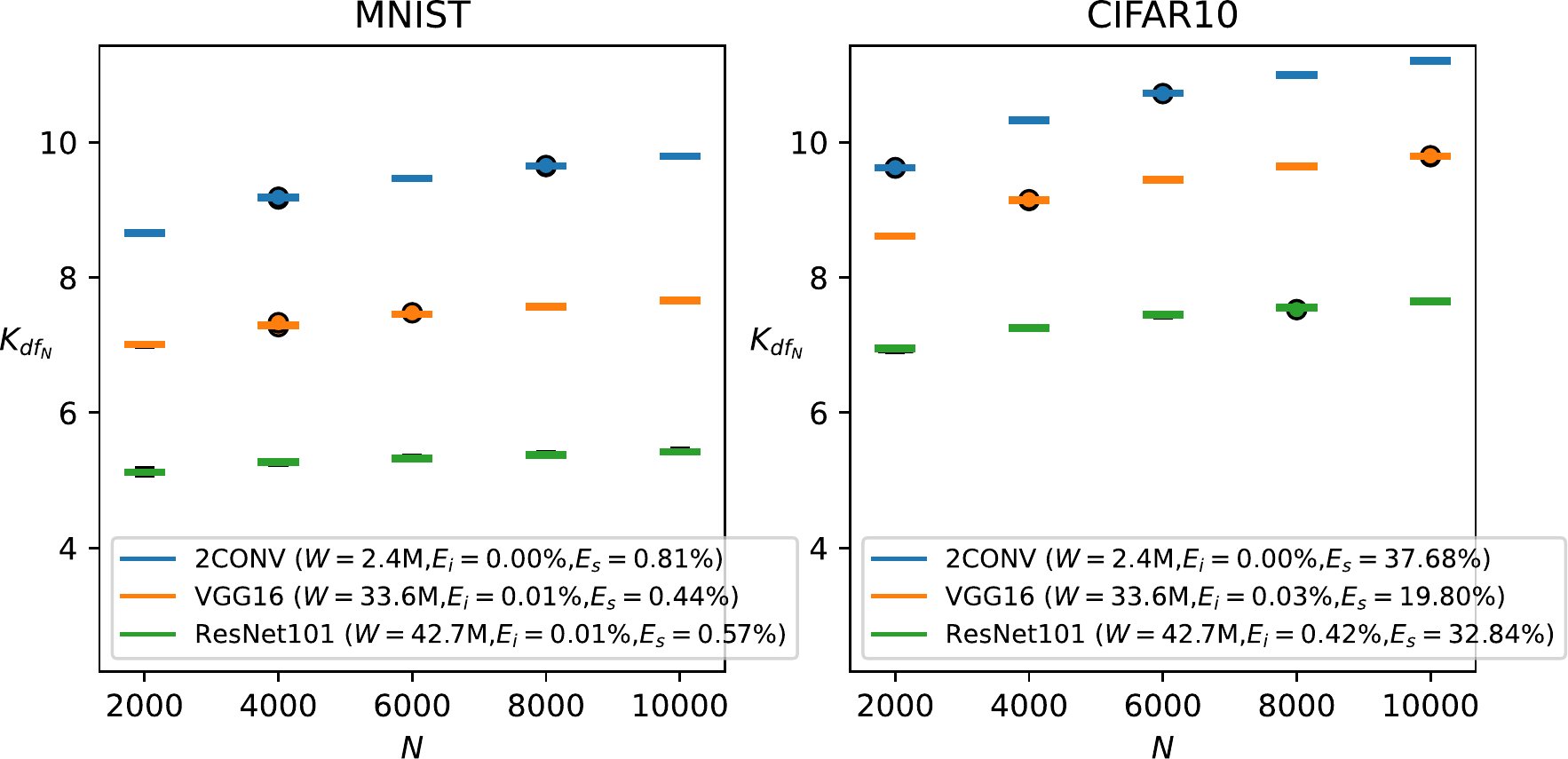}
    \caption{Conceptual capacity $\TBen_{df_N}$ over $N$ for 2CONV, VGG16 and ResNet101 architectures trained on MNIST (left plot) and CIFAR10 (right plot) datasets;  the box plots show the capacity over 5 evaluations of different choices of $N$ training samples from the corresponding dataset; the size $W$ (in millions of parameters), the percentage training error $E_i$ and test error $E_s$ of each network are shown in the legend.}\label{fig:bignet}
\end{figure*}

In this experiment we compared the conceptual capacity of different size networks trained on MNIST and CIFAR10 datasets using parameter values from the best performing epoch on the corresponding validation set.  In addition to VGG16 (33.6 million parameters) and ResNet101 (42.7 million parameters), we included a smaller two convolutional layer neural network we refer to as 2CONV (2.4 million parameters) (for description of the 2CONV architecture as well as the details of the training see Appendix \ref{app:bignet}).     
Figure \ref{fig:bignet} shows the results of the evaluation of $\TBen_{df_N}$ from MNIST and CIFAR10 datasets across all three architectures.   Just like in the previous experiment, we compute $\TBen_{df_N}$ 5 times for a given $N$, each time selecting a different random sample of inputs from the training set, but the same points between calculations of $\TBen_{df_N}$ for different architectures.  The shown $\TBen_{df_N}$ was averaged from all the evaluations over given $N$.  

This experiment shows that conceptual capacity does not necessarily grow with the network size, thus satisfying \citet{Neyshabur.etal:2017}'s \conref{\ref{cond:bignet}} condition.  We make a note of the fact that in this case the lowest conceptual capacity does not entail the best generalisation -- ResNet101 does worse than VGG16 on the test set for both MNIST and CIFAR10 experiments.  To explain this, we make an analogy to Support Vector Machines (SVMs), where maximising the margin reduces the complexity of the learner for a given choice of a kernel function.  Comparing the margin values from SVMs with different kernel functions is not sensible, as they operate in different kernel spaces.  Similarly, different networks might produce different concept spaces and their relative capacity might not provide any information about relative generalisation capabilities.  However, as we demonstrate in the next experiment, conceptual capacity within a given abstract space/network architecture does correlate with generalisation capabilities.  

\subsection{Correlation with generalisation}\label{sec:gen}

\begin{figure*}
    \centering
    \includegraphics[width=0.8\textwidth]{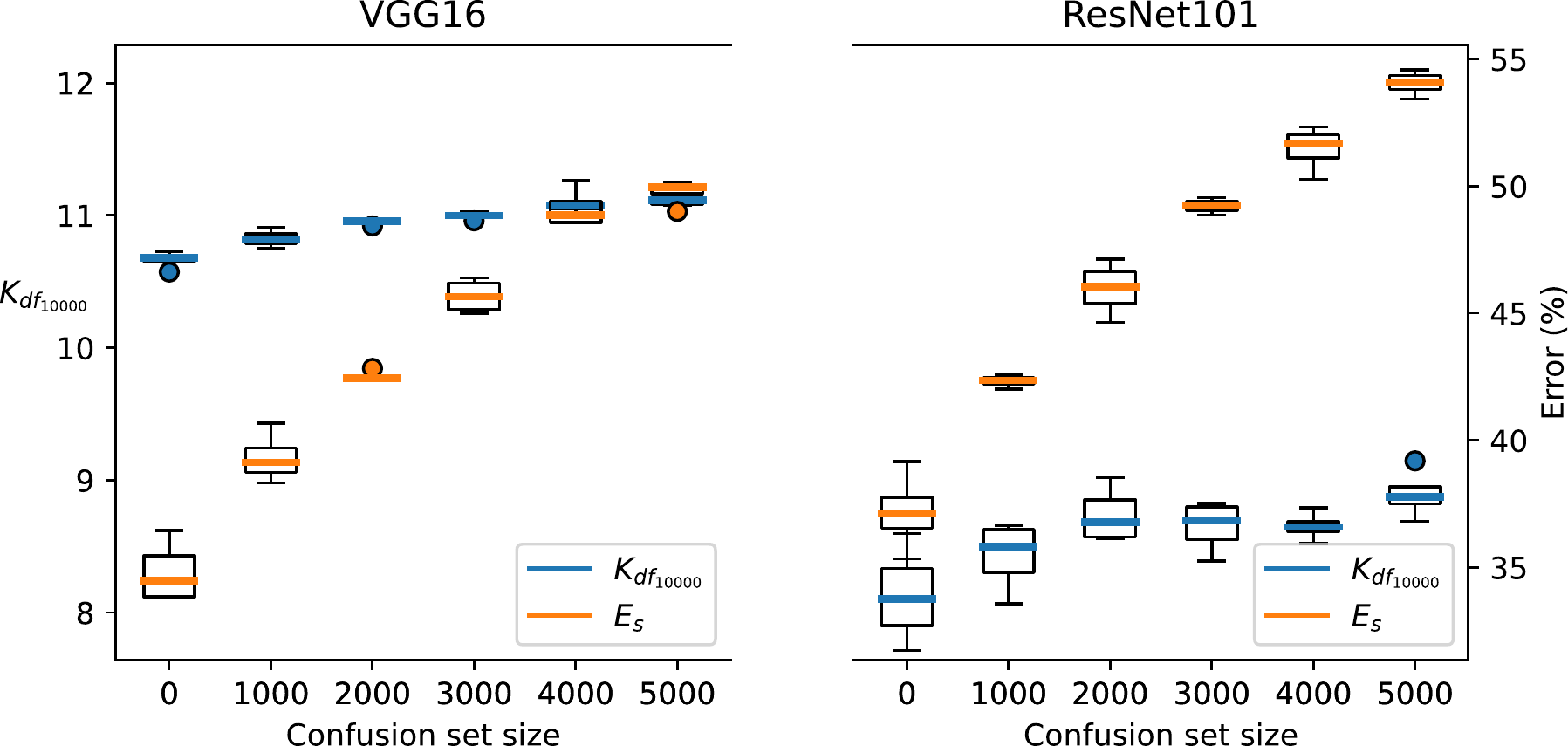}
    \caption{Box plot of conceptual capacity $\TBen_{df_{10000}}$ (blue lines over the left vertical axis) and test error (orange lines over the right vertical axis) from 4 different instances of VGG16 (left plot) and ResNET101 (right plot) trained to zero error on the CIFAR10-10K dataset combined with a confusion set of different size.}\label{fig:gen}
\end{figure*}

To test for the correlation between the conceptual capacity and generalisation error, we replicated the experiment from \citet{Neyshabur.etal:2017}.  We trained a set of networks on a subset of 10000 points from the CIFAR10 dataset, referred to as CIFAR10-10K, each time adding a confusion set of images with randomised labels.  Increasing the size of the confusion set while training to zero error forced the network into different minima with increasingly worse performance on the test set (for details of the training see Appendix \ref{app:gen}).   Figure \ref{fig:gen} shows the distribution of the results of the evaluation on the VGG16 and ResNet101-trained networks.   We ran the experiment 4 times for each size of the confusion set retaining the same CIFAR10-10K base.   Conceptual capacity was evaluated once for every network over the entire set of training points in the CIFAR10-10K subset.  As already evident in the previous experiment, ResNet101's conceptual capacity is generally lower than VGG16's, even though VGG16 generalises better on the MNIST and CIFAR10 data.  However, within a given architecture, as the test error increases, so does the conceptual capacity.

\begin{figure*}
    \centering
    \includegraphics[width=0.8\textwidth]{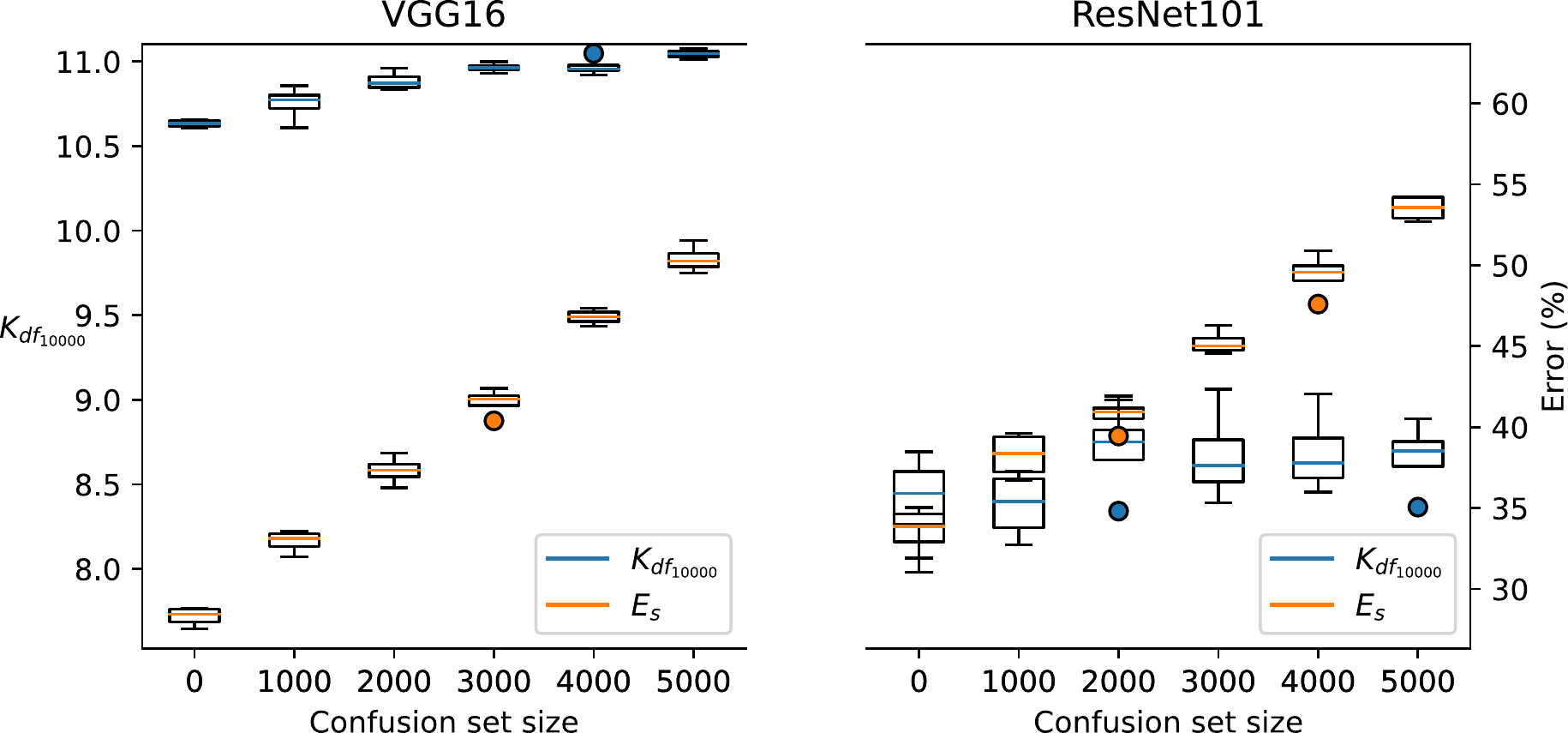}
    \caption{Box plot of conceptual capacity $\TBen_{df_{10000}}$ (blue lines over the left vertical axis) and test error (orange lines over the right vertical axis) from 4 different instances of VGG16 (left plot) and ResNet101 (right plot) trained to zero error on the CIFAR10-15MK dataset made up of confusion set of different size, but the total number of training points remaining the same.}\label{fig:gen2}
\end{figure*}

One potential confounding factor of \citet{Neyshabur.etal:2017}'s experiment is that the overall number of points in the training set changes as we change the confusion set size -- from 10K points (when confusion set size is 0) to 15K points (when the confusion set size is 5000).  The reason why conceptual capacity grows with test error might be simply because there are more inputs to memorise during training.  To account for this, we rerun the experiment topping up the $10000+M$ training points, $M$ being the confusion set size, with random selection of true labelled CIFAR10 images so that the total of training points is always 15K.  We refer to this training dataset as CIFAR10-15MK and the results of the experiment are shown in Figure \ref{fig:gen2}.  Though ResNet101's results are somewhat noisy, the correlation between conceptual capacity and training error for a fixed architecture is still there, thus satisfying \citet{Neyshabur.etal:2017}'s condition \ref{cond:gen}.

\subsection{Non-ReLU networks}\label{sec:sigmoid}

\begin{figure*}
    \centering
    \includegraphics[width=0.8\textwidth]{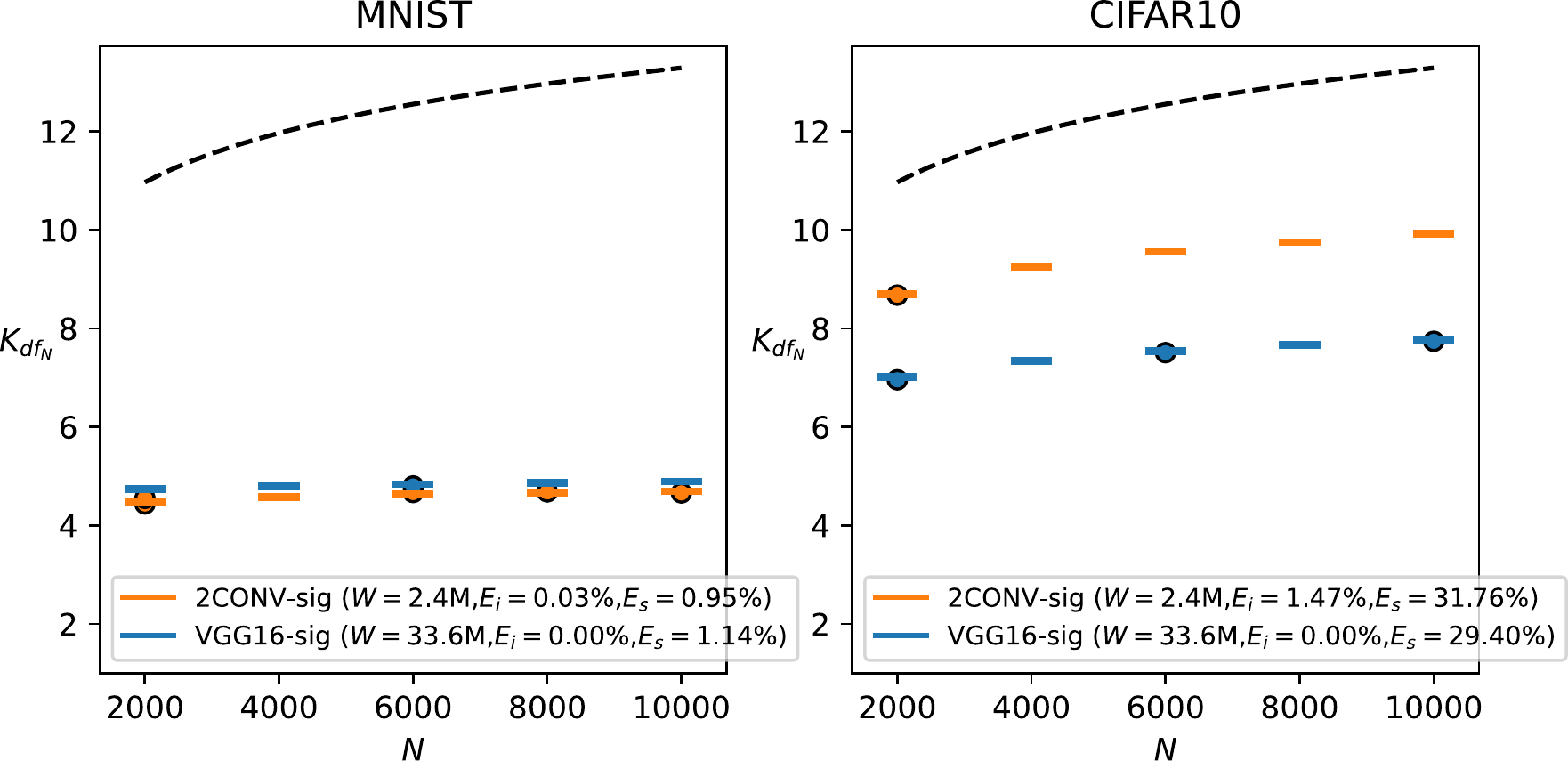}
    \caption{Conceptual capacity from 5 evaluations of $\TBen_{df_N}$ for different choices of $N$ training samples on 2CONV-sig and VGG16-sig networks; the practical upper bound on $\TBen_{df_N}$ is shown as a dashed black line; on the left is the analysis of the networks trained on MNIST, on the right the networks trained on CIFAR10; the size (in millions of parameters), the training and test error (in percent) of each network are shown in the legend as $W$, $E_i$ and $E_s$ respectively.}\label{fig:sigmoid}
\end{figure*}

Conceptual capacity can be evaluated on non-ReLU networks.  Without the exclusive use of the ReLU activation there is no upper bound on the size of the differential bundle $|df|$ and thus no theoretical upper bound on $\TBen_{df}$.  In practice however, there is a bound on $\TBen_{df_N}$, because $N$ differentials derived from $N$ inputs cannot represent more than $N$ independent concepts.  Following the logic of Lemma \ref{prop:hupper}, we have that

\begin{equation}
\TBen_{df_N}\le \log_2 N.
\end{equation}

In this evaluation we demonstrate that the conceptual capacity of neural networks with non-ReLU units after training is well below this bound, suggesting the diversity of the differential bundle is limited by the training. Figure \ref{fig:sigmoid} shows the result of the conceptual capacity analysis on 2CONV-sig and VGG16-sig, the networks from previous experiments with sigmoid activation in their fully connected layers.  The networks were trained on the MNIST and CIFAR-10 datasets (details of the training process are given in Appendix \ref{app:sigmoid}).  Once again, we ran 5 evaluations of $\TBen_{df_N}$ for a given $N$, each time selecting different random sample of inputs, but keeping the same points between calculations for different architectures.  

Interestingly, for this particular experiment, lower conceptual capacity does seem to correlate with better generalisation across different architectures.  This might be due to similarity between the 2CONV-sig and VGG16-sig's concepts spaces. Note that in Figure \ref{fig:bignet}, if we ignore ResNet101, lower $\TBen_{df_N}$ is also better for generalisation.  We believe that 2CONV and VGG16 architecture might be similar enough (no skip connections, just different hyperparameters) that their respective concept spaces end up being somewhat similar (see Section \ref{sec:acd} for some evidence on this), and thus comparing their relative conceptual capacities might be meaningful.  It was not possible to include ResNet101-based network in this evaluation, as we could not get it to train through a sigmoid activation in its latter layer(s). 


\section{Applications}\label{sec:app}

The most obvious, and arguably most pragmatic, application of conceptual capacity would be to use it as a regulariser during training in order to improve generalisation, in the hope that it would do for neural networks what maximum margin does for SVMs.  At this point it is not clear whether this would be possible and practical from a computational point of view.  Neither is it clear whether such regularisation would be any more effective than the existing regularisation techniques.  Nevertheless, right now, we can demonstrate several other uses of conceptual analysis as a tool for studying neural network models.  

\subsection{Training analysis}\label{sec:atrain}

\begin{figure*}
    \centering
    \includegraphics[width=0.9\textwidth]{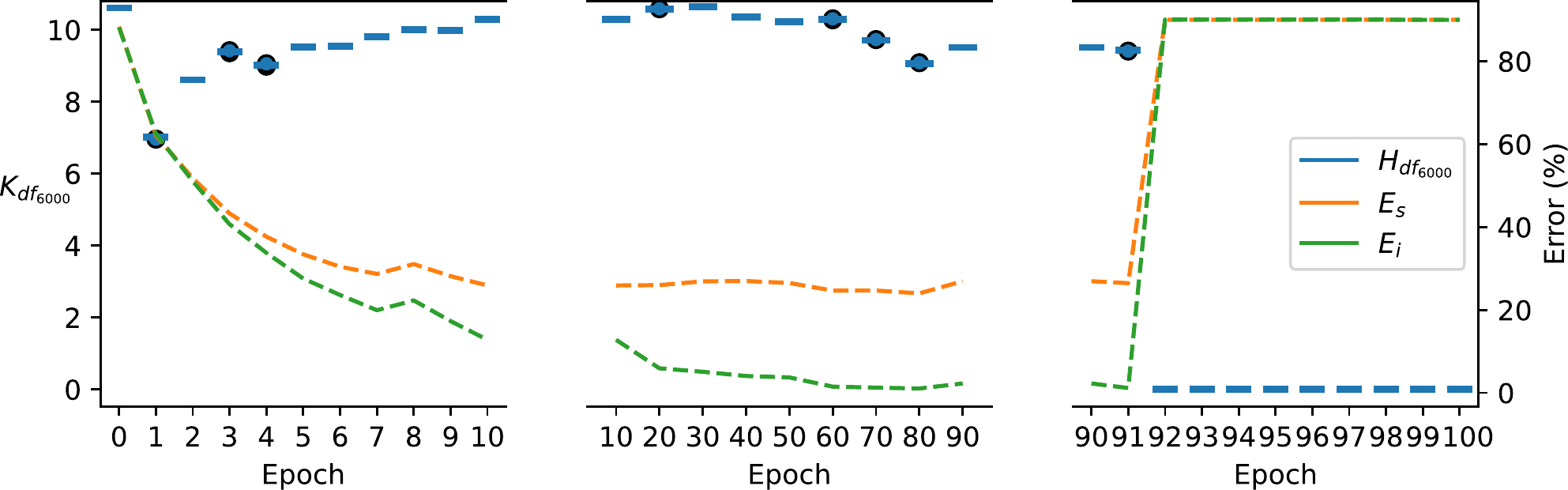}
    \caption{Conceptual capacity from 5 evaluations of $\TBen_{df_{6000}}$ over different sets of 6000 training points on VGG16 at different epochs of training a single instance on the CIFAR10 dataset;  conceptual capacity is shown as box plot (blue lines over the left vertical axis), the corresponding train and test errors are shown as dashed lines (green and orange over the right vertical axis).}\label{fig:atrain}
\end{figure*}

We can evaluate the conceptual capacity of a model during training, to see what happens to the architecture as it is learning a task.  Figure \ref{fig:atrain} shows such analysis of a VGG16 network at different stages of training on the CIFAR10 dataset (for details of the training see Appending \ref{app:atrain}).  We computed and averaged $\TBen_{df_{6000}}$ from 5 different choices of $N=6000$ training points for each epoch, using the same 5 sets of points for evaluations of conceptual capacity between the epochs.  At epoch 0, with weights from random initialisation, the conceptual capacity of the network is relatively high.  As soon as the training starts, $\TBen_{df_N}$ drops, suggesting the network is finding common patterns in the data.  Then $\TBen_{df_N}$ begins to rise when the network starts expanding its conceptual space in order to correctly label the training data.  Note the drop in conceptual capacity between epochs 60 and 80, which correlates with best generalisation.  After epoch 91, conceptual capacity plummets to 0 coinciding with the moment when the Adam \cite{Diederik.etal:2015} optimiser gets into an unstable state and the network loses generalisation.  It looks like it falls into a one concept configuration, suggesting the ReLU neurons have shifted into a permanent activation state, not switching between their linear and non-linear states for different input.

\subsection{Shallow vs. deep}\label{sec:adeep}

\begin{figure*}
    \centering
    \includegraphics[width=0.8\textwidth]{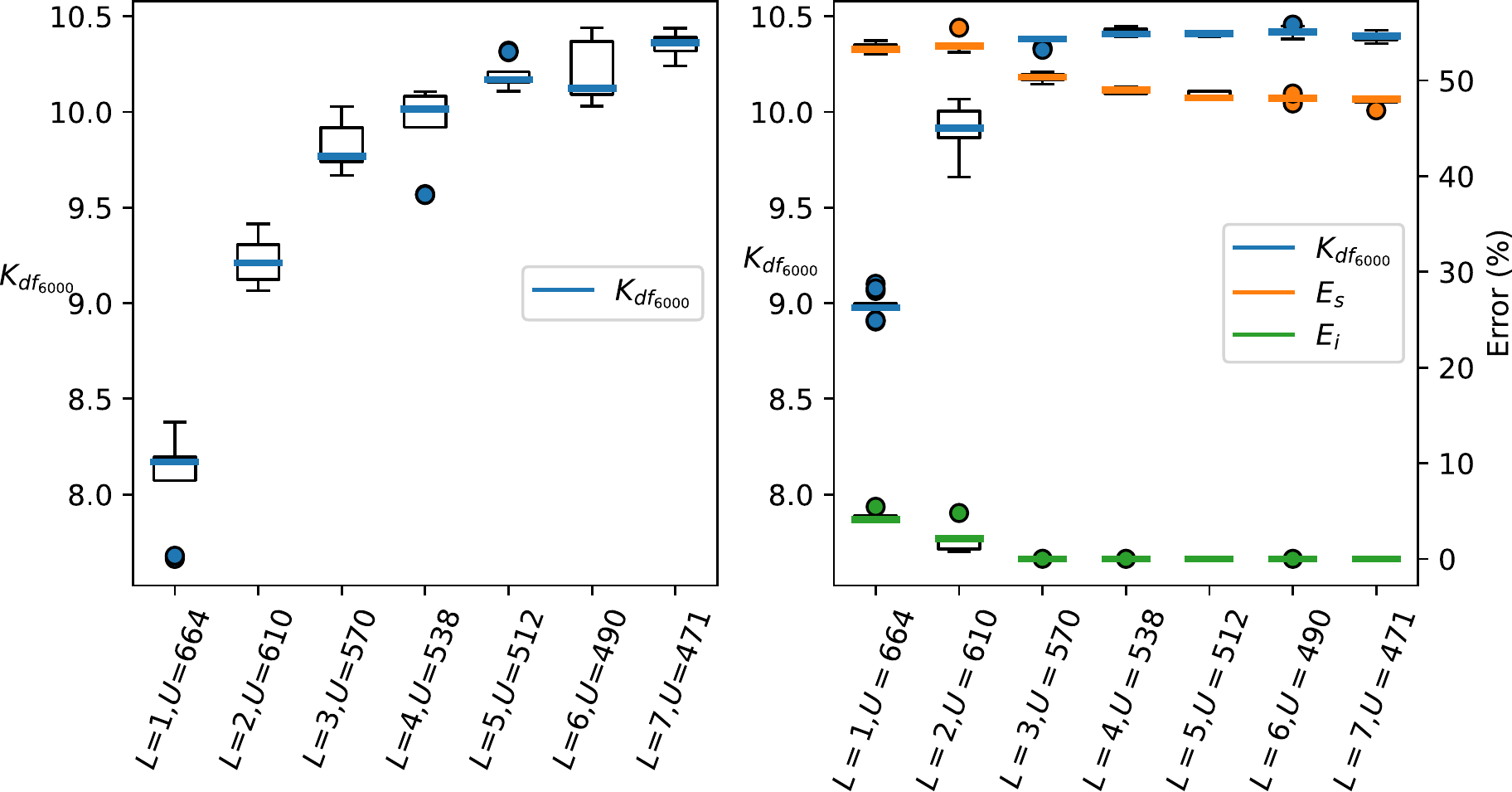}
    \caption{Box plot of conceptual capacity $\TBen_{df_{6000}}$ evaluated 3 times on 6000 different points from the CIFAR10 dataset on 5 randomly initialised not-trained (on the left) and trained (right) fully connected ReLU networks with increasing number of layers $L$ and decreasing number of neurons $U$; $U$ was chosen for given $L$ such that the total number of parameters in the network was $W\approx 4.6M$; the right-hand side plot includes the box plots of the distribution of the corresponding train and test error evaluation.}\label{fig:adeep}
\end{figure*}

Curious about the effect of the architecture on the initial conceptual capacity we devised an experiment highlighting differences between shallow and deep architectures.  We created a set of fully connected neural networks of different depth $L$, with input size for a CIFAR10 image and the number of neurons per layer changing with $L$ such that the total number of parameters remains constant.  Figure \ref{fig:adeep} shows the distribution of conceptual capacity over 5 instances of  random-weight networks for a given depth $L$ before and after training on CIFAR10 data -- in each case $\TBen_{df_{6000}}$ was computed 3 times over different choices of $N=6000$ input samples from CIFAR10 training set (for details of the training see Appending \ref{app:adeep})

On average, depth adds initial complexity to the random weight model.  After training, for $L\ge 3$ the complexity holds steady while the training error goes (marginally) down as the depth increases.  This suggests that deeper architecture does not simply reduce memorisation, but rather memorise concepts more suitable for generalisation.  Note that for $L<3$, the conceptual capacity is low because the network's internal representation does not even reach the complexity required to perfectly fit the training data.   


\subsection{Comparing two networks}\label{sec:acd}

\begin{figure*}
    \centering
    \includegraphics[width=0.9\textwidth]{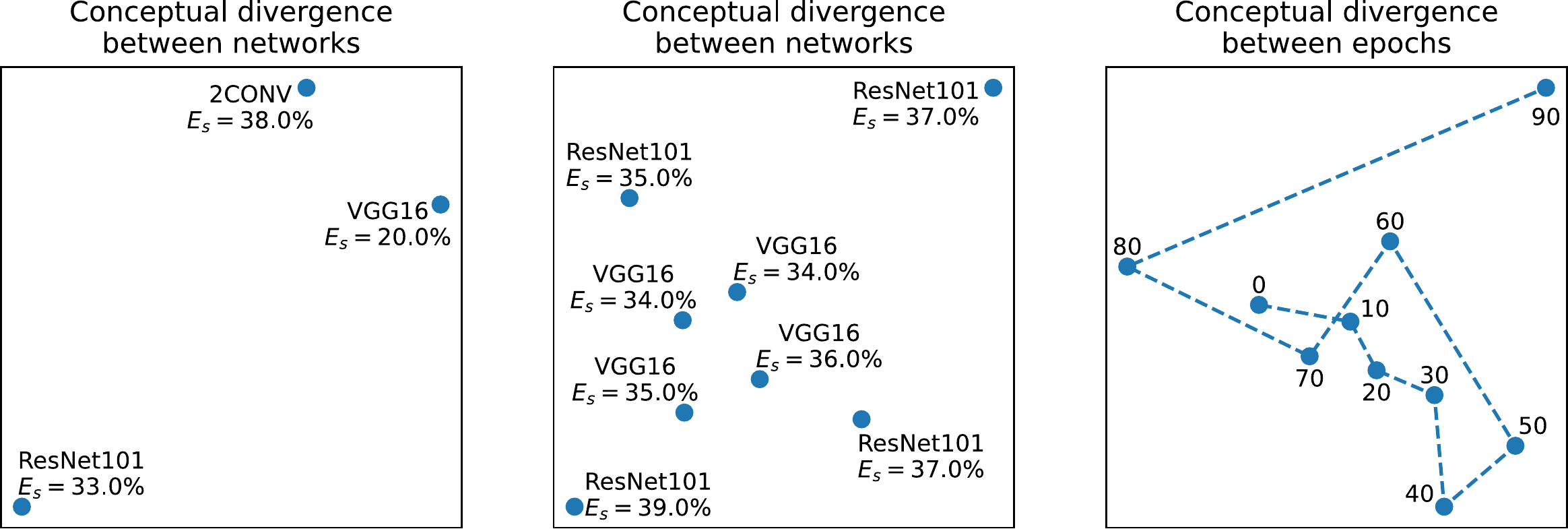}
    \caption{MDS visualisation of relative closeness of internal representation from a $|\CD(f,g)|$-based distance matrix for: the 2CONV, VGG16 and ResNet101 networks from the experiment in Section \ref{sec:bignet} trained on the CIFAR10 dataset (left); a set of VGG16 and ResNet101 networks from experiment in Section \ref{sec:gen} trained to zero error on the CIFAR10-10K dataset (middle); VGG16 network from Section \ref{sec:atrain} at different stages of training on CIFAR10 dataset (right); the left and middle image display the network name and corresponding test error by each point; the points in the right image are labelled with the epoch number (0 being the network at initialisation) with additional dashed line to reinforce the sequence of training.}\label{fig:acd}
\end{figure*}

Our framework allows for measurements of the entropy of the difference between two conceptual spaces, thus providing a tool for comparing the relative complexity of the internal representations of different networks.  Conceptual capacity can be generalised using the quantum Jensen-Shannon divergence \citep{briet2009properties}. Given two networks denoted by $f$ and $g$, the Jensen-Shannon divergence is defined based on their respective similarity matrices.

\begin{definition}[Conceptual divergence]
Let $S_{df_N}$, $S_{dg_N}$ be the similarity matrices for networks $f$ and $g$ respectively, evaluated on the same $N$ data points, and, with some abuse of notation, let $(f+g)$ be a (imaginary) network that would generate the similarity matrix $S_{d(f+g)_N} = \frac{1}{2}(S_{df_N}+S_{dg_N})$, then
\begin{equation*}
    \CD(f,g) = \TBen_{d(f+g)_N} - \frac{1}{2}(\TBen_{df_N}+\TBen_{dg_N}).
\end{equation*}
\end{definition}

Figure \ref{fig:acd} shows Multidimensional Scaling (MDS) representation of relative distance between a set of neural networks (each of specific weight and bias values) based on the pairwise absolute value of $\CD(f,g)$.  The left-hand side figure shows relative distance between networks from the experiment in Section \ref{sec:bignet} trained on the CIFAR10 dataset based on $\TBen{df_{6000}}$ for one particular choice of the $N=6000$ inputs.  The fact that 2CONV and VGG16 are much closer to each other than to ResNet101 supports our hypothesis that their conceptual spaces are similar (see discussion at the end of Section \ref{sec:sigmoid}).  

The middle plot of Figure \ref{fig:acd} shows visualisation of relative conceptual divergence between four instances of VGG16 and four instances of the ResNet101 neural networks from the experiment in Section \ref{sec:gen} trained on the CIFAR10-10K dataset to zero error.  Conceptual divergence computation was done based on $\TBen_{df_{10000}}$ over all the training points in the CIFAR10-10K dataset.  It suggest that, for the most part, internal representation of ResNet101 after training on this dataset is as distant from other ResNet101's as it is from VGG16's.  This explains a little bit why relationship between generalisation and conceptual capacity of ResNet101 shown in Figure \ref{fig:gen} is so much more noisy than that of VGG16's.  

The right-hand side plot of Figure \ref{fig:acd} shows an MDS visualisation based on conceptual divergence between different epochs of training of the VGG16 neural network from the experiment in Section \ref{sec:atrain} based on $\TBen_{df_{6000}}$ for one particular choice of the $N=6000$ points.  Network at epoch 100 was not included in this visualisation, as it was so far from all other points, that it made the instances from all the other epochs cluster into one location.

\section{Conclusion}

A ReLU neural network is essentially a switching circuit, with its neurons being forced into information passing or blocking state by the given input.  Taking the differential bundle of the network function as the manifestation of different switched configurations and treating each of those configurations as a separate PAC concept, we have proposed an entropy based measure of the network's conceptual capacity.  It can be thought of as a measure of memorisation of abstract concepts that are combined to make up the internal representation of a neural network with a particular set of weights and bias values. 

Conceptual capacity gives us a glimpse of why overparameterised neural networks generalise well -- during training, the activity of internal neurons becomes correlated and the conceptual capacity restricted for a given dataset.  Though there is theoretical potential for massive memorisation and the networks retain as many switching configurations as the number of training points in the dataset, the diversity between the PAC concepts can be quite small resulting in a relatively simple input-output mapping.  Hence we propose to use conceptual capacity as a measure of effective complexity of neural networks.  

The are still some limits to conceptual capacity.  First and foremost it is a data dependent measure, though we maintain that this is not a significant issue when evaluating on i.i.d. training data that is a good representative sample of the input-output mappings, as already required for good generalisation.  It is not an absolute measure, in that comparing conceptual capacity between architectures might not be meaningful, if those architectures happen to have very different conceptual spaces.  However, it is possible to measure the conceptual divergence between different networks, which at least informs when their conceptual spaces are very distinct.  Computationally the method is quite slow and memory intensive, especially during the phase of building and the eigen-decomposition of the similarity matrix $\Simf_{df_N}$.  However, as our evaluations show, evaluating over relatively small $N$, especially for the purpose of comparing the complexity, is sufficient.   Lastly, given the relatively small resolution of $K_{df_N}$ on ResNet101 evaluations, we note that $K_{df_N}$ might be a fairly crude measure of effective complexity.  As it is, it only gives information about the volume of the concept space ignoring how individual PAC concepts from different switch configurations of the network are distributed in that space.  We hypothesise that consideration of the clustering tendencies of the PAC concepts within the concept space, in addition to information about the volume, might provide more accurate information on a network's effective complexity.

\section*{Acknowledgement}
We gratefully acknowledge the support of NVIDIA Corporation with the donation of the TITAN X GPU used for this research.


\appendix
\section{Training details}
\label{app:training}

Details of the experiments from Sections \ref{sec:eval} and \ref{sec:app}.  In order to make MNIST and CIFAR10 suitable for VGG16 and ResNet101, and maintain consistency of the results of data-dependent $\TBen_{df_N}$ across different experiments, all the images were resized to $48\times 48$ pixels and normalised to $[-128,128]$ range.  MNIST dataset was divided into 55000 training, 5000 validation and 10000 test images.  CIFAR10 dataset was divided into 45000 training, 5000 validation and 10000 test images.  Training data was  randomly shuffled in each epoch before division into mini-batches.  Unless otherwise stated, the loss function was softmax.

\renewcommand\thesubsection{A.%
\arabic{subsection}}

\subsection{\nameref{sec:toy} training details}\label{app:toy}

The three binary class training sets used for this experiments are shown in Figure \ref{fig:toy}; input was normalised to range [0,1].  We trained a fully connected, 3 hidden layer, 128 neurons per layer ReLU network on each dataset optimising cross-entropy loss using Adam optimiser with default Tensorflow values for $\beta_1=0.9$, $\beta_2=0.999$, $\epsilon=10^{-7}$ and with learning rate of $0.001$.  Training was done over 1000 epochs using all 120 samples for single batch-based update per epoch.  In each instance the network trained to zero-training error.  
\subsection{\nameref{sec:random} experiment training details}\label{app:random}

All the models were trained using stochastic gradient descent optimiser with learning momentum of 0.9 and exponential decay of the learning rate.  VGG16 and ResNet101 training on the MNIST and MNIST-RND datasets was done over maximum of 400 epochs with mini-batch size of 128, starting learning rate of $0.01$ exponentially decayed step-wise every 40 epochs at a rate of $0.631$.  VGG16 training on the CIFAR10 and CIFAR10-RND datasets was done over maximum of 400 epochs with mini-batch size of 128, starting learning rate of $0.001$ exponentially decayed step-wise every 40 epochs at a rate of $0.398$.   ResNet101 training on the CIFAR10 and CIFAR10-RND datasets was done over maximum of 600 epochs with mini-batch size of 512, starting learning rate of $0.001$ exponentially decayed step-wise every 60 epochs at a rate of $0.398$.  In every case the training was halted as soon as zero training error was reached.  In cases when the target training error was not reached within designated maximum number of epochs, the training was rerun with new random starting weights.  

\subsection{\nameref{sec:bignet} experiment training details}\label{app:bignet}

The 2CONV network architecture was inspired by an old Tensorflow \citep{abadi2016tensorflow} convnet tutorial for MNIST.  Our version of the architectures consisted of two convolutional layers of 32 neurons/filters taking input from a $5\times 5$ input window with step of 1 and `same' padding, each followed by a max pooling layer with $3\times 3$ window input, step of 2 and `same' padding.  These were followed by a 512-neuron fully connected layer before the final output layer.  All the neurons used ReLU activation.  In all the instances training was run over 200 epochs using Adam optimiser with learning rate of $0.001$ and mini-batches of 128 samples.  The final weights of the network were chosen from the best performing epoch in terms of the validation error.  

\subsection{\nameref{sec:gen} experiment training details}\label{app:gen}

The CIFAR10-10K subset of 10000 images with true CIFAR10 labels was combined in turn with a confusion set of 0, 1000, 2000, 3000, 4000, and 5000 other images from the CIFAR10 dataset with randomised labels.  We trained VGG16 and ResNet101 models on each combination of the true+confusion dataset.  This was repeated 4 times, keeping the true label set the same and but randomly choosing a different confusion set. VGG16 training was done with the stochastic gradient optimiser, batch size of 128, starting learning rate of 0.01 decayed every 20 epochs step-wise with rate of 0.5.  Training was halted when the training error reached zero and the loss was less than $10^{-4}$, otherwise, after 200 epochs, the lowest loss weights that give zero training error were used.  If zero error was not obtained after 200 epochs, training was restarted with new randomly initialised weights.  ResNet101 training was done with the Adam optimiser for a maximum of 400 epochs, batch size of 64, starting learning rate of 0.01 decayed every 40 epochs step-wise with rate of 0.5.  Training was halted when the training error reached zero, otherwise, after 400 epochs it was restarted with new randomly initialised weights.

The CIFAR10-15MK experiment was done exactly under the same conditions as CIFAR10-10K, except we added 5000, 4000, 3000, 2000, 1000, and 0 images with their true labels (from the CIFAR10 dataset) to the respective confusion set of 0, 1000, 2000, 3000, 4000, and 5000.

\subsection{\nameref{sec:sigmoid} experiment training details}\label{app:sigmoid}

We trained the 2CONV-sig and VGG16-sig networks, which are just 2CONV and VGG16 architectures with sigmoid activation in their fully connected hidden layers,  on the MNIST and CIFAR10 datasets.  We used Adam optimiser with constant learning rate of $0.001$ and batch size 128.  Networks were training for 200 epochs, at which point the weights that gave best accuracy on the validation set were selected.   

\subsection{\nameref{sec:adeep} experiment training details}\label{app:adeep}

In all the instances training was run over 200 epochs with mini-batches of 128 using Adam optimiser with starting learning rate of $0.0001$ exponentially decayed step-wise every 20 epochs at a rate of $0.631$.   The final weights of the network were chosen from the best performing epoch in terms of the validation error.  

\subsection{\nameref{sec:atrain} experiment training details}\label{app:atrain}

We trained VGG16 network on CIFAR10 dataset using Adam optimiser with constant learning rate of $0.001$ and batch size of 128 saving network parameters at every epoch.

\vskip 0.2in
\bibliography{references}

\end{document}